\documentclass[10pt,twocolumn,letterpaper]{article}

\usepackage[pagenumbers]{cvpr} 

\definecolor{cvprblue}{rgb}{0.21,0.49,0.74}
\usepackage[pagebackref,breaklinks,colorlinks,allcolors=cvprblue]{hyperref}

\usepackage[utf8]{inputenc} 
\usepackage[T1]{fontenc}    
\usepackage{hyperref}       
\usepackage{url}            
\usepackage{booktabs}       
\usepackage{amsfonts}       
\usepackage{nicefrac}       
\usepackage{microtype}      
\usepackage{xcolor}         
\usepackage{amsmath, amssymb}
\usepackage{graphicx}
\usepackage{arydshln}
\usepackage{amssymb}
\usepackage{amsmath}
\usepackage{amsthm}
\usepackage{lipsum} 
\usepackage{algorithm}
\usepackage{algorithmicx}
\usepackage{algpseudocode}  
\usepackage{enumitem}
\usepackage{multirow}
\usepackage{subcaption}
\usepackage{pifont}
\usepackage{bm}
\usepackage{breqn}
\usepackage{makecell}
\usepackage{arydshln}
\usepackage{svg}

\newtheorem{proposition}{Proposition}

\def\paperID{*****} 
\def\confName{CVPR}
\def\confYear{2026}

\title{Robust Self-Training with Closed-loop Label Correction \\ for Learning from Noisy Labels}

\author{
  Zhanhui Lin \\
  \and
  Yanlin Liu \\
  \and
  Sanping Zhou \\
}

\newcommand{\x}{x}
\newcommand{\y}{y}
\newcommand{\z}{z}
\newcommand{\fdim}{d'}
\newcommand{\meta}{m}
\newcommand{\xmeta}{x^{\meta}}
\newcommand{\ymeta}{y^{\meta}}
\newcommand{\zmeta}{z^{\meta}}
\newcommand{\ytilde}{\tilde{y}}
\newcommand{\yhat}{\hat{y}}
\newcommand{\ybar}{\bar{y}}

\newcommand{\ybarmeta}{\ybar^{\meta}}

\newcommand{\ytildehatmeta}{\hat{\ytilde}^{\meta}}

\newcommand{\Dtrain}{\tilde{\mathcal{D}}}
\newcommand{\Dmeta}{\mathcal{D}^{\text{m}}}
\newcommand{\Dtrainm}{\mathcal{D}^{\text{m}}_{\text{train}}}
\newcommand{\Dvalm}{\mathcal{D}^{\text{m}}_{\text{val}}}
\newcommand{\Ntrain}{N}
\newcommand{\Nmeta}{M}

\newcommand{\Loss}{\mathcal{L}}

\DeclareMathOperator*{\argmin}{arg\,min}   

\begin{document}

\maketitle

\begin{abstract}
Training deep neural networks with noisy labels remains a significant challenge, often leading to degraded performance. Existing methods for handling label noise typically rely on either transition matrix, noise detection, or meta-learning techniques, but they often exhibit low utilization efficiency of noisy samples and incur high computational costs. In this paper, we propose a self-training label correction framework using decoupled bilevel optimization, where a classifier and neural correction function co-evolve. Leveraging a small clean dataset, our method employs noisy posterior simulation and intermediate features to transfer ground-truth knowledge, forming a closed-loop feedback system that prevents error amplification.
Theoretical guarantees underpin the stability of our approach, and extensive experiments on benchmark datasets like CIFAR and Clothing1M confirm state-of-the-art performance with reduced training time, highlighting its practical applicability for learning from noisy labels.
\end{abstract}

\section{Introduction}
\label{sec:intro}

The remarkable success of deep neural networks (DNNs) in various domains has highlighted the critical importance of large-scale datasets for achieving high-performance models. However, obtaining completely clean training data is often prohibitively expensive in real-world applications. When DNNs are trained on noisy datasets, they tend to memorize noise rather than learn meaningful patterns, leading to performance degradation and reduced generalization capability \citep{degrade, degrade2}. This phenomenon has spurred extensive research into robust training methods that can effectively handle noisy labels, giving rise to the field of learning with noisy labels (LNL).

While fully clean datasets are scarce, accurately annotating a small subset of data remains feasible. Research has shown that even a limited amount of clean data can significantly enhance model performance \citep{distill}, as it provides reliable supervision that anchors learning and corrects noise-induced biases. This observation has motivated several approaches in the LNL literature.
Transition matrix-based methods attempt to estimate the mapping between noisy and clean labels using clean data. Despite progress, these approaches struggle with instance-dependent noise patterns \citep{yang2021estimating, instance-dependent} and suffer from high variance that affects neural network learning. Meta-learning approaches address this challenge by optimizing hyperparameters on small-scale clean data through bilevel optimization. However, these methods suffer from interpretability deficiencies, and their inner-loop virtual updates significantly increase memory usage and computational cost \citep{mloc, metacleaner}, limiting their application to large-scale scenarios. Other methods apply clean samples for noise detection to identify contaminated data and train with reduced weights or discard them entirely \citep{shafir2025active, yu2023delving}. However, these approaches are typically multi-stage, time-consuming, and underutilize noisy samples through weighting or discarding.

To address these limitations, we propose a novel self-training label correction framework that enables the noisy label correction function and the main classifier to co-evolve synergistically without error amplification.
Whereas existing methods often rely on final output pseudo-labels, which may be inaccurate and limit the capacity to model complex noise patterns, our approach leverages intermediate feature representations that encapsulate richer information about the underlying data.
We employ noisy posterior simulation to facilitate the transfer of ground-truth knowledge from the clean dataset to the refinement of noisy labels via the correction function. As the classifier is trained on progressively higher-quality corrected labels, its feature extractor generates representations that increasingly approximate the true semantic distribution of the data, thereby further enhancing the performance of the correction function.

Our approach offers several key advantages over existing methods. It explicitly leverages clean data as feedback during training, thereby mitigating cumulative error amplification and circumventing the complex threshold selection issues prevalent in traditional self-training methods that rely on heuristic designs \cite{Song2019SELF, Zheng2022SELC}.
Unlike weighting or noise detection approaches, our method corrects rather than down-weights or discards noisy labels, thereby achieving higher utilization of noisy data while reducing variance.
Whereas existing literature often necessitates additional self-supervised learning \cite{Tu2023DMLP, emlc} or a dual model \cite{wang2024pspu} to generate pseudo-labels, our correction process operates in a considerably more lightweight and stable manner.
Moreover, compared to state-of-the-art meta-learning methods, our approach utilizes clean data in a distinct, decoupled fashion, achieving superior performance with lower computational complexity and enhanced suitability for distributed training.

\paragraph{Contributions} Our main contributions are three-fold. First, we propose a self-training framework where classifier and correction function co-optimize synergistically, using clean labels to guide the self-rectification on incorrect labels and prevent drift from ground truth. Second, we provide a theoretical guarantee showing that model performance is maintained or improved even when suboptimal new labels are introduced, confirming the stability of our method during noise correction. Third, experiments on CIFAR and Clothing1M datasets demonstrate that our method achieves performance breakthroughs while reducing training time, validating its practical value in large-scale scenarios with limited clean data.

\section{Related Work} \label{sec:related}

\paragraph{Strategies for Learning with Noisy Labels} The field of Learning with Noisy Labels (LNL) has developed diverse strategies to combat label noise, which can otherwise cause deep networks to memorize errors and generalize poorly \cite{Arpit2017Closer, Li2024NoisySurvey}. Broadly, these methods include: (1) Loss-based strategies, which design robust loss functions (e.g., GCE \cite{Zhang2018Generalized}) or estimate a noise transition matrix (NTM) to correct the loss \cite{Xia2019Anchor}; (2) Sample selection, which identifies and down-weights or removes likely noisy samples, often using techniques like co-teaching \cite{Han2018Coteaching} or mixture models (e.g., DivideMix \cite{Li2020DivideMix}); and (3) Label refurbishment, which attempts to directly repair noisy labels using model predictions \cite{Lyu2024RobustLogit}. Regularization, such as consistency training and data augmentation (e.g., MixUp \cite{Li2020DivideMix}), is also a critical component in robust training.

\paragraph{Self-Training Methods for LNL} Self-training methods have gained prominence in LNL by iteratively refining noisy labels using the model's own predictions as pseudo-labels, often in a semi-supervised manner. For instance, Noisy Student \cite{Xie2020NoisyStudent} iteratively trains student models on pseudo-labeled data augmented with noise, improving robustness through teacher-student distillation. Similarly, SELF \cite{Nguyen2020SELF} employs self-ensembling to filter noisy labels, leveraging ensemble predictions to enhance pseudo-label quality. ProSelfLC \cite{Yao2021ProSelfLC} progressively corrects labels by trusting the model's own predictions more over time, exploiting the early-learning phenomenon to mitigate noise memorization. More recent approaches, such as SELC \cite{Zheng2022SELC}, use self-ensemble techniques for label correction in supervised settings, while methods like DISC \cite{Li2023DISC} dynamically select and correct instances based on model confidence.
These methods are often limited in their capacity to model complex noise patterns. Moreover, they frequently rely on heuristic strategies, such as voting or thresholding, which lack theoretical guarantees and can lead to model collapse due to iterative error propagation.

\paragraph{LNL with a Limited Clean Dataset} The availability of a small, trusted clean dataset, as explored in this work, significantly bolsters LNL. This clean set is often used in several ways. Some treat the problem as semi-supervised learning \cite{Song2019SELF, Zhu2025ZMT}. Meta-learning approaches utilize the clean data as a meta-validation set to learn sample re-weighting functions \cite{mwnet, metacleaner} or pseudo-label bootstrapping weights \cite{l2b, metadistill}. While effective, these methods often introduce high computational and memory overhead due to their coupled, inner-loop optimization. Alternatively, the clean set can facilitate direct noise estimation and correction. This includes using it to estimate a global NTM \cite{glc, fasten} or to provide anchors for label refurbishment \cite{Lyu2024RobustLogit}. Our work falls into this category but proposes a decoupled, co-evolutionary framework to overcome the limitations of both heuristic self-training and costly meta-learning.
    
\begin{figure*}[t]
\centering
\vspace{-10pt}
\includegraphics[width=0.92\linewidth]{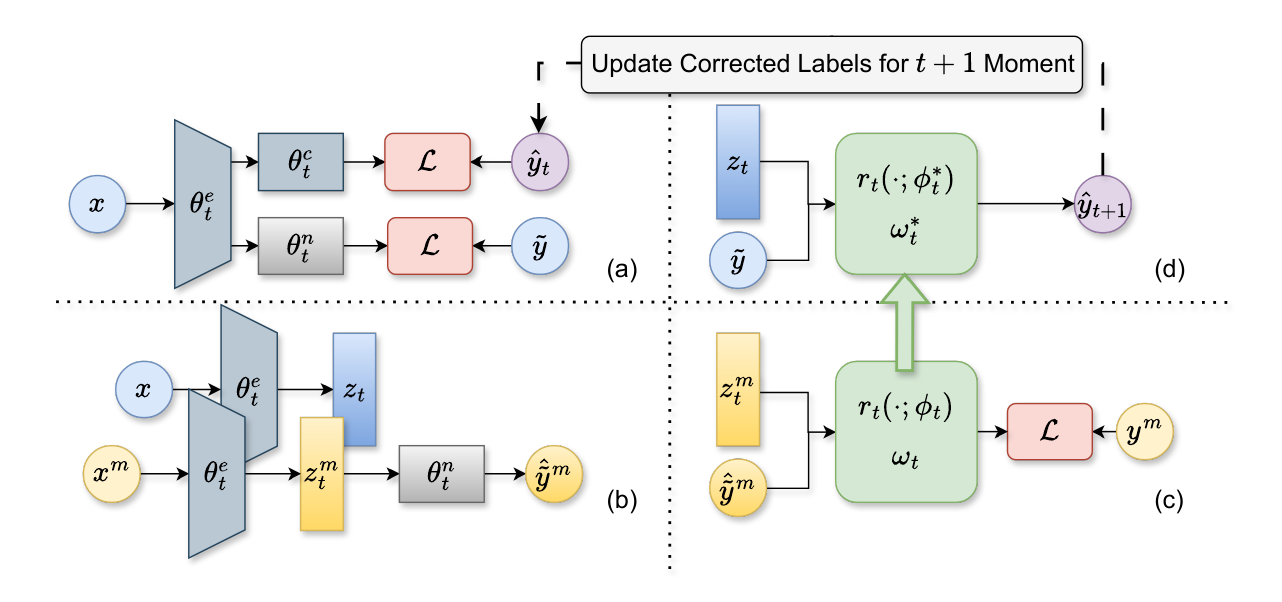}
\vspace{-15pt}
\caption{This framework provides a general overview of the approach. (a) Training the classifier on the corrected labels. (b) Feature collection and simulation of the clean-data noisy posterior. (c) Optimization of the correction function. (d) Noisy label correction.}
\label{fig:method}
\end{figure*}

\section{Preliminaries}
\label{sec:pre}

In this study, we address the problem of $c$-class classification. The input feature space is represented as $\x \in \mathbb{R}^d$, and the label space is denoted as $\y \in \mathbb{R}^c$, where $c$ signifies the number of classes.
We define the noisy training dataset as $\Dtrain = \{(\x_{i}, \ytilde_{i})\}_{{i}=1}^{\Ntrain}$, where each $\ytilde_{i}$ is an observed noisy label.
We also incorporate a subset of trusted data, referred to as the clean validation dataset $\Dmeta = \{(\xmeta_{i}, \ymeta_{i})\}_{{i}=1}^{\Nmeta}$.
Usually, $M \ll N$ holds, as is common in real-world scenarios.

Let \( f(\x; \theta) \) be a neural network model parameterized by \( \theta \) for classification. 
We assume that our classification model comprises a feature extractor $e(\mathbf{x}; \theta^e)$ and a prediction head $g(\mathbf{z}; \theta^c)$. The extractor $e(\mathbf{x}; \theta^e)$ maps an input $\mathbf{x} \in \mathbb{R}^{d}$ to a feature vector $\mathbf{z} \in \mathbb{R}^{d'}$ (where $d' \ll d$). Subsequently, $g(\mathbf{z}; \theta^c)$, parameterized by $\theta^c$, produces class probabilities via a softmax layer, yielding $f(\mathbf{x}; \theta^e, \theta^c) = g(e(\mathbf{x}; \theta^e); \theta^c)$ to approximate $p(y | \mathbf{x})$. 
We also augment this model with an auxiliary "noisy" prediction head $g_n(\mathbf{z}; \theta^n)$ as detailed in \cref{sec:noisy_posterior_simulation}. This noisy head shares the feature extractor $e(\mathbf{x}; \theta^e)$ and possesses an identical architecture to $g(\cdot; \theta^c)$, yielding $f_n(\mathbf{x}; \theta^e, \theta^n) = g_n(e(\mathbf{x}; \theta^e); \theta^n)$ to approximate $p(\ytilde | \mathbf{x})$. The complete model parameters are thus $\theta = \{\theta^e, \theta^c, \theta^n\}$. 
Our goal is to train this augmented model to achieve strong generalization on clean data with $f(\mathbf{x}; \theta^e, \theta^c)$, even when primarily learning from noisy labels.

\newcounter{proprep}
\setcounter{proprep}{0}

\section{Methods}
\label{sec:methods}

Our method is an iterative self-training framework that jointly optimizes a \textit{primary classifier} and a \textit{label correction mechanism}. This creates a synergistic loop: the correction mechanism, learned on the clean dataset $\Dmeta$, refines noisy labels $\ytilde$ from $\Dtrain$ into cleaner estimates $\yhat$; the main classifier $f$ then trains on these $\yhat$; an improved classifier subsequently helps the correction mechanism produce better $\yhat$ for the next iteration. 

We summarize our process in \cref{fig:method}. Within each iteration with timestamp $t$, key operations include: (a) updating the classifier model parameters $(\theta^e_t, \theta^c_t, \theta^n_t)$ using $\yhat_t$ (for $f$) and $\ytilde$ (for $f_n$); (b) collecting features with $e(\cdot, \theta^e_t)$ on $\Dtrain$ and  $\Dmeta$, simulating noisy posteriors via $f_n$ ; (c) optimizing the label correction mechanism using $\Dmeta$ along with its simulated noisy posteriors and collected features; (d) finally, generating new corrected labels $\yhat_{t+1}$ for samples in $\Dtrain$ to guide the subsequent classifier training.

\subsection{Noisy Posterior Simulation}
\label{sec:noisy_posterior_simulation}

To learn how to correct noisy labels back to clean ones, we leverage the clean dataset $\Dmeta$ where ground truth is available. However, $\Dmeta$ contains only clean labels without corresponding noisy versions, presenting a fundamental challenge: \textbf{how can we train a correction function without noisy-clean label pairs?}
To address this issue, we introduce noisy posterior simulation to generate synthetic noisy labels for the clean dataset. We train a ``noisy simulation'' on $\Dtrain$ to approximate $p(\ytilde|x)$ and then apply it to $\Dmeta$ to estimate $p(\ytilde^m|\xmeta)$. This approach builds a bridge between noisy and clean label distributions, enabling the knowledge from $\Dtrain$ to correct labels in $\Dtrain$ despite working with entirely different instances.

We implement this simulation using a two-head architecture similar to \cite{fasten} for efficient training. The noisy classifier shares the feature extractor $e(\cdot; \theta^e)$ with $f$, but uses a different prediction head which is parameterized with $\theta^n$. This forms a noisy prediction model $f_n(x; \theta) = g_n(e(x; \theta^e); \theta^n)$. 
The noisy classifier is explicitly trained to predict the observed noisy labels in $\Dtrain$. This is accomplished by jointly optimizing $\theta_t$ at iteration $t$ using the composite loss function:
\begin{equation}
\label{eq:framework_train_f}
\theta_t = \argmin_{\theta} \frac{1}{\Ntrain} \sum_{i=1}^{\Ntrain} \Big( \Loss(f(x_i; \theta), \hat{y}_{t,i}) + \lambda \Loss(f_n(x_i; \theta), \tilde{y}_{i}) \Big),
\end{equation}
where $\lambda$ is a hyperparameter to balance these objectives and prevent overfitting to noise.
The first term directs the clean classifier $f$ towards the corrected clean labels, while the second term compels the noisy classifier $f_n$ to model the characteristics of the actual noisy labels.

After updating $\theta_t$, the simulated noisy posterior for each clean sample $\xmeta \in \Dmeta$ is generated as $\ytildehatmeta = f_n(\xmeta; \theta_t)$. This simulated posterior $\ytildehatmeta$ serves as the indispensable noise-informative input during the process on the clean dataset in the following sections.

\subsection{Robust Closed-loop Label Correction}
\paragraph{Neural Label Correction}
\label{sec:neural_correction}

Traditional transition-matrix-based approaches struggle to model complex, instance-dependent label noise and often suffer from high estimation variance. Neural networks, in contrast, offer greater expressive capacity and lower variance, making them better suited for this task. We therefore propose a neural correction module, $r$, parameterized by $\phi$, to leverage these advantages.

A key design choice for this module is its use of intermediate feature representations. This is inspired by the early learning phenomenon - where networks capture clean patterns before memorizing noise, a principle underlying pseudo-labeling. We posit that these intermediate features retain more stable and generalizable semantics than final-layer logits. This richer information is hypothesized to enable the correction function $r$ to better infer potential sources of classification error.
Based on this rationale, the module $r$ is designed to take two inputs: (i) the noisy label posterior ($\tilde{y}$ for $\mathcal{D}_{\text{train}}$ or simulated $\ytildehatmeta$ for $\Dmeta$) and (ii) the aforementioned intermediate feature representation ($z$ or $\zmeta$). The correction network processes an input of dimension $c + d_f$, where $d_f$ is the feature dimensionality, and outputs a $c$-dimensional corrected label distribution $\bar{y} = r(\tilde{y}, z; \phi)$ via a softmax layer.

Despite this, training $r$ on the typically small meta set $\Dmeta$ presents a significant overfitting risk. To mitigate this, $r$ is intentionally implemented as a shallow network (e.g., a two-layer MLP in our experiment). To further address this, we partition $\Dmeta$ into a training set $\mathcal{D}_{\text{train\_m}}$ and a validation set $\mathcal{D}_{\text{val\_m}}$. The parameters $\phi$ are optimized by minimizing a meta-loss on $\mathcal{D}_{\text{train\_m}}$, and the final parameters $\phi^*$ are selected based on performance on the held-out validation set $\mathcal{D}_{\text{val\_m}}$, underpinning the robust label correction process.

\paragraph{Robust Convex Combination}
\label{sec:convex_combination}

We define $\ybar_t = r_t(\ytilde, \z_t; \phi_t^*)$ as the prediction of the neural correction network $r_t$ at iteration $t$, and $\ybarmeta_t = r_t(\ytildehatmeta, \zmeta_t; \phi_t^*)$ as its prediction on the clean validation set.
Despite careful training and validation, the neural correction network $r_t(\cdot; \phi_t^*)$ optimized at any given iteration $t$ might still produce suboptimal predictions $\ybar_t$. Relying solely on the latest prediction $\ybar_t$ could be risky: \textbf{if $\ybar_t$ is flawed, using it directly as the target label $\yhat_t$ could amplify errors and destabilize the classifier training.}

To enhance robustness, we form an ensemble of historical predictions. We define the original noisy label as the 0-th correction, $\ybar_0 \equiv \ytilde$. The corrected label after iteration $t$ is then a convex combination of all predictions $\{\ybar_k\}_{k=0}^t$, weighted by a vector $\omega_t \in \Delta^t$, where $\Delta^t$ is the $t$-simplex (i.e., $\sum_{k=0}^t \omega_{t,k} = 1$ and $\omega_{t,k} \geq 0$):
\begin{equation} \label{eq:full_convex_combo}
\yhat_{t+1} = \sum_{k=0}^t \omega_{t,k} \ybar_k.
\end{equation}
\begin{proposition} \label{prop:non_increasing_risk}
Let $\hat{y}_K^* = \sum_{k=0}^K \omega_k^* \bar{y}_k$ for $K \geq 0$ (with $\ybar_0 \equiv \ytilde$), where the weight vector $\omega^* \in \Delta^K$ minimizes the expected risk $\mathcal{R}(\hat{y}_K^*)$. Then:
\begin{align}
\mathcal{R}(\hat{y}_K^*) &\leq \min_{0 \leq k \leq K} \mathcal{R}(\bar{y}_k), \\
\mathcal{R}(\hat{y}_{K}^*) &\leq \mathcal{R}(\hat{y}_{K-1}^*) \quad \text{for all} \quad K \geq 1.
\end{align}
That is, the expected risk $\mathcal{R}(\hat{y}_K^*)$ is bounded by the minimum risk of any individual component (including the original noisy label $\ybar_0$) and does not increase as $K$ increases.
\end{proposition}
These properties are vital to prevent error amplification, especially when individual corrected predictions $\bar{y}_k$ might be suboptimal. 
They adaptively down-weights unreliable $\bar{y}_k$ while enabling uncertain predictions (e.g., even near-uniform) to induce beneficial label smoothing. This yields a stable, low-variance training signal  and implicit regularization, ensuring robustness even under suboptimal corrections.

Since the correction network has already overfitted to $\Dtrainm$, we approximate $\omega_{t}^*$ by minimizing the meta-loss on $\Dvalm$. We similarly define $\ybarmeta_0 \equiv \ytildehatmeta$ for the meta set:
\begin{equation} \label{eq:weight_optimization}
\omega_{t}^* = \argmin_{\omega \in \Delta^t} \frac{1}{|\Dvalm|} \sum_{i \in \Dvalm} \Loss \left( \sum_{k=0}^t \omega_{t,k} \ybarmeta_{k,i}, \ymeta_i \right).
\end{equation}

Due to the low dimensionality of $\omega$ and the averaging effect of convex combination, overfitting is mitigated even with small $\Dvalm$, producing high-quality corrected labels $\yhat_{t+1}$ for classifier training in the subsequent iteration.

\begingroup
\begin{algorithm}[tbp]
    \caption{Self-Training with Closed-loop Label Correction Algorithm}
    \label{algo:classification}
    \begin{algorithmic}[1]
    \setlength{\abovedisplayskip}{2pt}
    \setlength{\belowdisplayskip}{3pt}
    \setlength{\jot}{1pt}
        \Require Noisy training set $\Dtrain = \{(\x_i, \ytilde_i)\}_{i=1}^{\Ntrain}$; clean set $\Dmeta = \{(\xmeta_i, \ymeta_i)\}_{i=1}^{\Nmeta}$.
        \State Set $t \gets 0$; Initialize $\hat{y}_{0,i} \gets \ytilde_i$ for $i=1, \ldots, \Ntrain$.
        \State Split $\Dmeta$ into $\Dtrainm$ and $\Dvalm$.
        \Repeat
            \State Update $\theta_t$ using corrected labels $\yhat_{t}$ and noisy labels $\ytilde$ via \cref{eq:framework_train_f} for several epochs.
            \State Extract the features on $\Dtrain$ and $\Dmeta$ and simulate noisy posteriors on $\Dmeta$.
            \State Define base corrections: $\ybar_0 \gets \ytilde$ and $\ybarmeta_0 \gets \ytildehatmeta$.
            \State Initialize and optimize correction network $r_t(\cdot; \phi_t)$ on $\Dtrainm$, validate on $\Dvalm$ to get $\phi_t^*$.
            \State Generate corrections with $r_t(\cdot; \phi_t^*)$ and optimize weights $\omega_t^*$ on $\Dvalm$ via \cref{eq:weight_optimization}.
            \State Update corrected labels $\hat{y}_{t+1}$ for $\Dtrain$ using $r_t(\cdot; \phi_t^*), \omega_{t}^*$ via \cref{eq:full_convex_combo}.
            \State Set $t \gets t + 1$.
        \Until{Stopping criterion met}
    \end{algorithmic}
\end{algorithm}
\endgroup

\subsection{Algorithm Details}
\label{sec:algorithm_details}
The overall iterative self-training process is detailed in \cref{algo:classification}. The algorithm begins with initialization and then iteratively updates the classifier parameters, extracts features, simulates noisy posteriors, and optimizes the correction network. In each iteration, it generates new predictions using the correction network, computes optimal combination weights, and updates the corrected labels for the next round. This process continues until a stopping criterion is met, progressively improving both the classifier and the quality of corrected labels. 

The computational cost associated with optimizing the correction network $r$ and the weights $\omega$ remains minimal. This is primarily due to the typically small size of the clean dataset $\Dmeta$, the shallow architecture of $r$ (e.g., a 2-layer MLP), and the significantly lower dimensionality of feature embeddings compared to the raw input data. In our CIFAR experiments, these optimization steps constituted only about 1\% of the total training time, ensuring that our method does not introduce a significant computational overhead.

\paragraph{Iterative Improvement Mechanism}
The efficacy of our method is significantly driven by its iterative improvement mechanism, which functions based on three interconnected principles:

\begin{itemize}
    \item Our mechanism establishes a closed-loop feedback system. The clean data provide essential ground-truth anchors, enabling the transfer of true label information to the noisy training data via the learned correction function and preventing error amplification.
    \item Our approach leverages the widely recognized early learning phenomenon that neural networks typically learn cleaner, more generalizable patterns from the data before beginning to memorize noisy instances, providing a strong foundation for the label correction module.
    \item A virtuous cycle emerges through iterative refinement, where increasingly accurate labels enhance the classifier’s feature representations, which in turn improve the label correction function. \cref{prop:non_increasing_risk} guarantees that this cycle maintains or improves performance,  ensuring the algorithm's stability.
\end{itemize}

\newcommand{\cmark}{\ding{51}}%
\newcommand{\xmark}{\ding{55}}%
\newcommand{\smark}{\ding{70}}%

\section{Experiments}
\label{sec:experiments}

\subsection{Experiments on CIFARs} 
\label{sec:experiments:dataset}
We first evaluate our method on CIFAR-10 and CIFAR-100  \cite{cifar}, which is widely used for machine learning research.
Both datasets contain 50,000 training and 10,000 test images. We randomly select 5,000 images from the training set to serve as the clean subset. This initial partitioning establishes a 9:1 ratio between the remaining noisy training data and the designated clean subset, which represents a realistic scenario of limited clean annotations. Furthermore, we design another experiment on CIFAR-100 to explore a more extreme challenge, utilizing an even smaller clean set of only 1,000 samples (49:1, using 10 samples per class), as detailed in \cref{sec:cn1000}.

\begin{table*}[t]
\centering
\setlength{\tabcolsep}{6.8pt} 
\renewcommand{\arraystretch}{0.98} 
\caption{Test accuracy (\%) of all competing methods on CIFAR-10 and CIFAR-100 under different noise types with different noise levels. The best results are highlighted in bold.}
\label{tab:main_results}
\resizebox{0.93\textwidth}{!}{%
\begin{tabular}{l@{\hspace{0.5em}}l@{\hspace{0.5em}}c|ccc|ccc|ccc}
\toprule
\multirow{3}{0.085 \textwidth}[0.6ex]{\flushright Dataset} & \multicolumn{2}{c|}{\multirow{3}{0.085 \textwidth}[0.6ex]{\flushright Method}} & \multicolumn{3}{c|}{Symmetric} & \multicolumn{3}{c|}{Asymmetric} & \multicolumn{3}{c}{Instance-dependent} \\
\cmidrule{4-12}
& & & \multicolumn{3}{c|}{$\eta$} & \multicolumn{3}{c|}{$\eta$} & \multicolumn{3}{c}{$\eta$} \\
\cmidrule{4-12}
& & & 0.2 & 0.4 & 0.6 & 0.2 & 0.4 & 0.6 & 0.2 & 0.4 & 0.6 \\
\midrule
\multirow{12}{*}[-2ex]{CIFAR-10} & \multirow{2}{*}{Cross-Entropy } & Best & 91.13 & 87.93 & 82.54 & 92.87 & 90.00 & 71.99 & 90.75 & 82.74 & 68.52 \\
& & Last & 82.84 & 71.30 & 59.93 & 88.45 & 79.92 & 70.19 & 78.15 & 54.54 & 39.72 \\
\cmidrule{2-12}
& \multirow{2}{*}{FasTEN} & Best & 91.46 & 88.42 & 82.88 & 92.87 & \textbf{92.68} & 91.39 & 91.19 & 85.44 & 77.09 \\
& & Last & 86.19 & 77.62 & 70.04 & 91.63 & 88.82 & 85.71 & 86.82 & 81.16 & 61.88 \\
\cmidrule{2-12}
& \multirow{2}{*}{MW-Net} & Best & 90.55 & 88.02 & 83.45 & 93.03 & 91.29 & 65.42 & 90.66 & 82.30 & 63.83 \\
& & Last & 84.86 & 72.58 & 57.04 & 90.49 & 87.21 & 61.65 & 79.29 & 66.94 & 52.41 \\
\cmidrule{2-12}
& \multirow{2}{*}{L2B} & Best & 89.15 & 86.27 & 80.38 & 91.37 & 90.56 & 89.66 & 89.67 & 85.98 & 74.23 \\
& & Last & 88.83 & 86.27 & 80.23 & 91.23 & 90.51 & 89.23 & 89.46 & 85.79 & 74.23 \\
\cmidrule{2-12}
& \multirow{2}{*}{EMLC} & Best & 90.77 & 87.84 & 82.00 & 91.67 & 90.92 & 73.08 & 83.63 & 79.35 & 72.35 \\
& & Last & 81.60 & 71.24 & 58.55 & 88.51 & 80.95 & 68.11 & 80.32 & 74.09 & 61.57 \\
\cmidrule{2-12}
& \multirow{2}{*}{{Ours}} & Best & \textbf{92.48} & \textbf{89.28} & \textbf{85.50} & \textbf{93.88} & 92.65 & \textbf{93.07} & \textbf{92.83} & \textbf{91.37} & \textbf{85.76} \\
& & Last & \textbf{92.12} & \textbf{88.73} & \textbf{84.61} & \textbf{93.70} & \textbf{92.61} & \textbf{92.97} & \textbf{92.69} & \textbf{91.09} & \textbf{85.24} \\
\midrule
\multirow{12}{*}[-2ex]{CIFAR-100} & \multirow{2}{*}{Cross-Entropy } & Best & 67.59 & 59.99 & 49.94 & 69.69 & 65.33 & 55.96 & 62.57 & 45.90 & 24.05 \\
& & Last & 62.93 & 45.80 & 26.81 & 62.80 & 46.62 & 30.75 & 60.01 & 40.09 & 20.88 \\
\cmidrule{2-12}
& \multirow{2}{*}{FasTEN} & Best & 68.09 & 62.27 & 50.42 & 67.64 & 53.42 & 31.50 & 66.57 & 59.29 & 47.39 \\
& & Last & 59.37 & 43.65 & 26.09 & 61.91 & 43.27 & 31.50 & 59.70 & 45.79 & 33.15 \\
\cmidrule{2-12}
& \multirow{2}{*}{MW-Net} & Best & 67.94 & 62.26 & \textbf{54.24} & 70.04 & 65.26 & 57.23 & 62.87 & 48.31 & 26.53 \\
& & Last & 62.38 & 48.59 & 29.55 & 62.87 & 47.82 & 31.95 & 59.25 & 40.91 & 22.75 \\
\cmidrule{2-12}
& \multirow{2}{*}{L2B} & Best & 64.57 & 54.54 & 20.66 & 64.91 & 58.42 & 43.33 & 66.27 & 59.55 & 47.45 \\
& & Last & 64.28 & 54.38 & 20.66 & 64.66 & 58.11 & 43.33 & 66.05 & 59.55 & 47.45 \\
\cmidrule{2-12}
& \multirow{2}{*}{EMLC} & Best & 68.67 & 62.49 & 46.90 & 70.10 & 65.06 & 57.15 & 63.82 & 53.99 & 32.11 \\
& & Last & 62.80 & 46.51 & 31.74 & 62.94 & 47.09 & 32.34 & 61.94 & 49.34 & 28.44 \\
\cmidrule{2-12}
& \multirow{2}{*}{{Ours}} & Best & \textbf{68.73} & \textbf{63.00} & 50.94 & \textbf{70.18} & \textbf{67.21} & \textbf{63.80} & \textbf{70.69} & \textbf{63.72} & \textbf{55.08} \\
& & Last & \textbf{68.47} & \textbf{62.36} & \textbf{50.74} & \textbf{69.76} & \textbf{66.76} & \textbf{63.22} & \textbf{70.64} & \textbf{63.58} & \textbf{54.97} \\
\bottomrule
\end{tabular}
}
\end{table*}

\paragraph{Noise settings}
Let $\eta$ denote the corruption ratio. We first simulate two types of instance-independent noise.
\textbf{Symmetric}: Generated by flipping the labels of a given proportion of training samples to all the other class labels uniformly.
Each instance with a true label $y$ is corrupted to all possible classes with probability $\frac{\eta }{c}$.
\textbf{Asymmetric}: Generated by flipping the labels of a given proportion to one or few particular classes.
For CIFAR-10, we disturb the label in its similar class with probability $\eta $ following \cite{safeguard}.
For CIFAR-100, we disturb the label to other classes in its super-class as described in \cite{glc}.
Noise transition matrices are generated first, and then labels are randomly corrupted based on the matrix for each instance-independent label noise setting.
Additionally, we test an \textbf{Instance-dependent} noise setting, simulated as described in \cite{instance-dependent}, applied to both datasets.
We vary $\eta $ in the range of $[0, 1]$ to simulate different noise levels for all noise types.

\paragraph{Baselines}
We compare our methods with a series of advanced methods in noisy label learning that use clean data. The compared methods include:
\textit{Transition-Matrix-Based Methods}:
\textbf{FasTEN.} \cite{fasten} introduce a two-head architecture to efficiently estimate the label transition matrix every iteration within a single back-propagation.
\textit{Meta-Learning-Based Methods}:
\textbf{MW-net.} \cite{mwnet} 
 uses an MLP net and meta-learns the weighting function.
\textbf{L2B.}  \cite{l2b} uses a bootstrapping method with meta-learned weighting coefficient.
\textbf{EMLC.} \cite{emlc} proposes a teacher-and-student architecture and meta-learns the taught soft-label.
We exclude label selection methods due to their high clean data requirements and low noise utilization.
For fair comparison, we disable EMLC's self-supervised initialization to focus comparisons on the core noise-handling mechanisms.

\paragraph{Implementations}
\label{sec:implementations}
In our experiments on CIFAR-10 and CIFAR-100, we employed ResNet-34 as the backbone architecture. 
The models were trained using the stochastic gradient descent (SGD) optimizer with momentum 0.9 and weight decay $5e^{-4}$ for 100 epochs, and the batch size was set to 128. We set the initial learning rate to 0.1 with a step decay scheduler that reduces the learning rate by a factor of 0.1 at epochs 60 and 80.
A full description of the experimental setup can be found in \cref{tab:experiment-details}.

For the specific settings of our method, we make the loss balancing factor $\lambda = 0.5$ in \cref{eq:framework_train_f} for simplicity. We divide \(\Dmeta\) into training \(\Dtrainm\)  and validation subsets \(\Dvalm\) using a ratio \(0.8:0.2\). The neural network correction model is a 2-layer MLP with 256 hidden units, ReLU activation, and a softmax output layer.
We use SGD with learning rate 0.001 (reduced to 0.0001 after first validation decline) on \(\Dtrainm\) to optimize the $\phi$, while we employ sequential least squares programming to optimize \(\omega\) on \(\Dvalm\). 
We update the features, correction function, and update the corrected labels every 5 epochs, after the initial 40-epoch warm-up, which is intended to allow for an adequately initialized noisy posterior. 
For a comprehensive analysis of our method's sensitivity to hyper-parameters, we present ablation studies in \cref{app:time_breakdown_cifar}, examining both the impact of varying $\lambda$ values and the effect of different correction function update frequencies.

\paragraph{Results} \label{sec:cifar-main-results}
The comparative performance analysis is presented in \cref{tab:main_results}, with demonstrative training dynamics for all methods on CIFAR-100 under instance-dependent noise at noise rate 0.4, illustrated in \cref{fig:cifar-dynamic}. Generally, our method consistently outperforms other approaches across all noise scenarios, demonstrating strong robustness to label noise. Notably, while other approaches show substantial performance degradation in later training epochs (likely attributable to overfitting on noisy patterns), our approach maintains stable performance trajectories without significant accuracy decline. This stability empirically validates that our correction mechanism effectively leverages the small clean dataset to guide the learning process, preventing memorization of noisy patterns.

\begin{figure*}[t]
    \begin{minipage}[t]{0.48\textwidth}
        \centering
        \includegraphics[width=\textwidth]{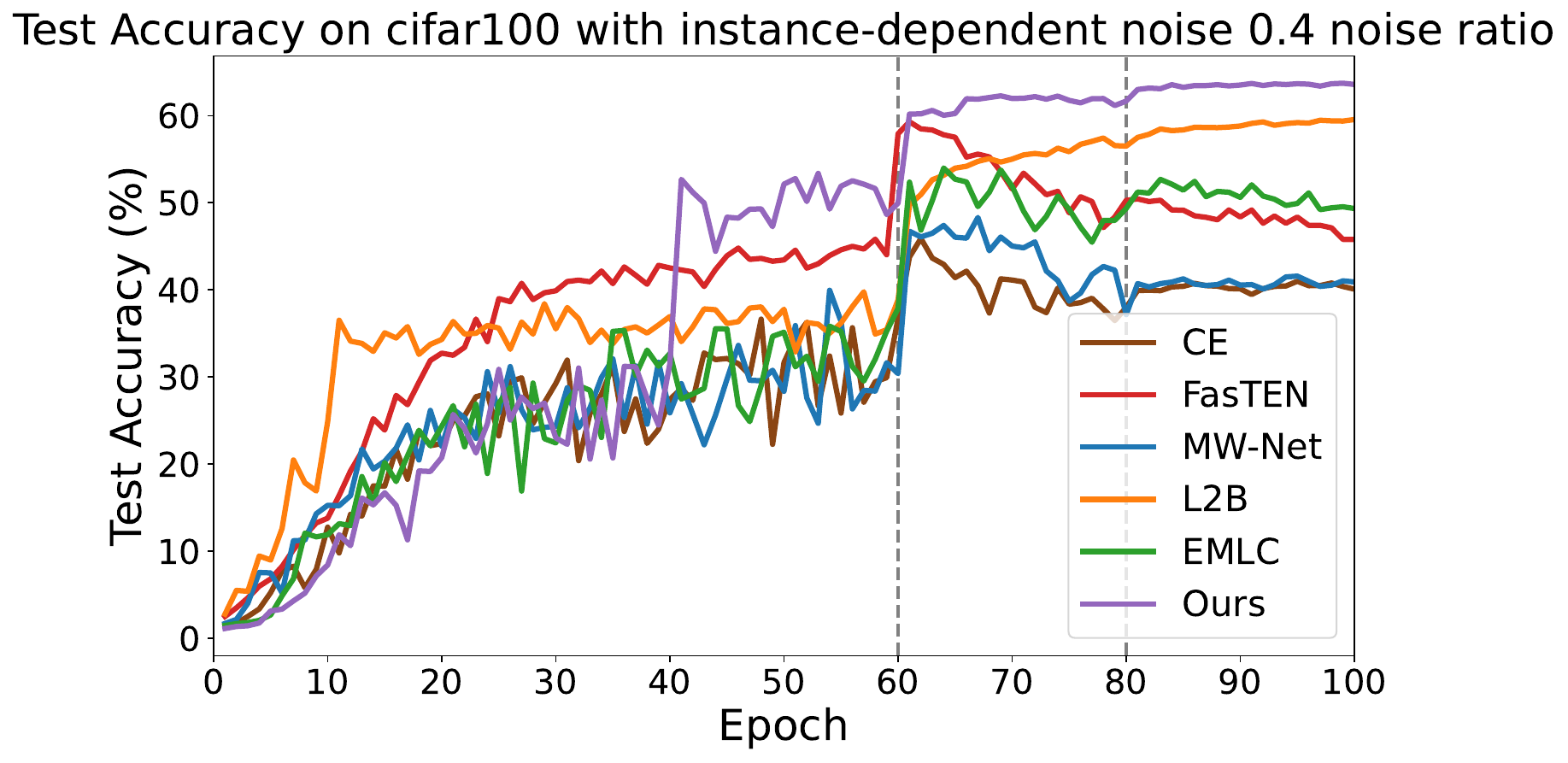}
        \caption{Training dynamics of different methods on CIFAR-100 under instance-dependent noise with a level of 0.4. The vertical dashed lines represent the epochs at which the learning rate decays at 60 and 80 epochs.}
        \label{fig:cifar-dynamic}
    \end{minipage}
    \hfill
    \begin{minipage}[t]{0.48\textwidth}
        \centering
        \includegraphics[width=\textwidth]{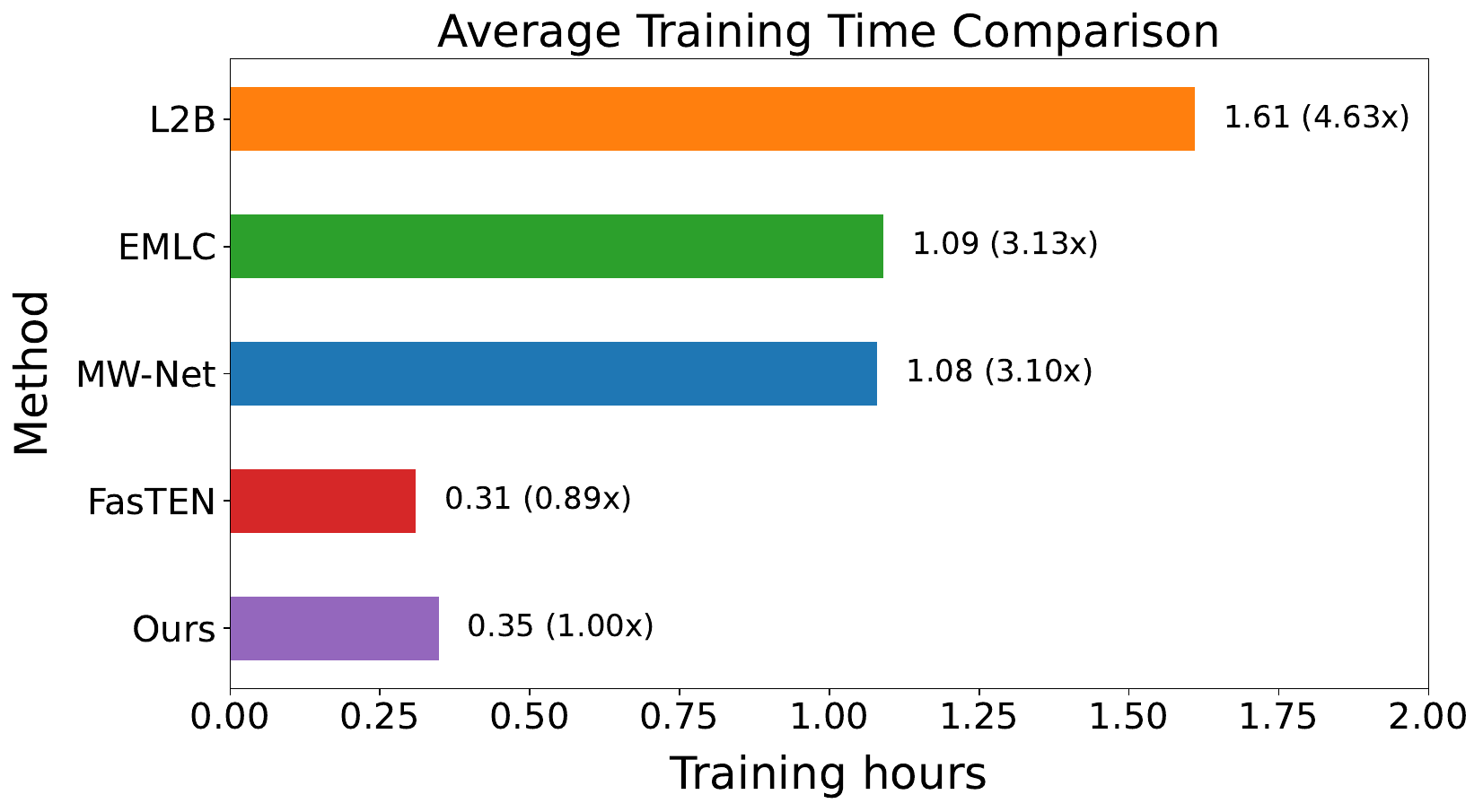}
        \caption{Comparison of average GPU training time on CIFAR-100 (in hours) for various methods using a single RTX 4090. The values in parentheses indicate the relative ratio of training time compared to our method. }
        \label{fig:cifar-time}
    \end{minipage}
\end{figure*}

\begin{table*}[t]
\centering
\small 
\setlength{\tabcolsep}{6.8pt} 
\renewcommand{\arraystretch}{0.98} 
\caption{Test accuracy (\%) of our method and its ablations on CIFAR-100. Setting (a) uses 5,000 clean samples (50 samples per class), and Setting (b) uses 1,000 clean samples (10 samples per class). The last recorded performances in the training procedure are reported.}
\label{tab:cifar-ablations}
\resizebox{0.93\textwidth}{!}{%
\begin{tabular}{c|c|ccc|ccc|ccc}
\toprule
\multirow{3}{*}[0.6ex]{Setting} & \multirow{3}{*}[0.6ex]{Ablation} & \multicolumn{3}{c|}{Symmetric} & \multicolumn{3}{c|}{Asymmetric} & \multicolumn{3}{c}{Instance-dependent} \\
\cmidrule(lr){3-5} \cmidrule(lr){6-8} \cmidrule(lr){9-11}
& & \multicolumn{3}{c|}{$\eta$} & \multicolumn{3}{c|}{$\eta$} & \multicolumn{3}{c}{$\eta$} \\
\cmidrule(lr){3-5} \cmidrule(lr){6-8} \cmidrule(lr){9-11}
& & 0.2 & 0.4 & 0.6 & 0.2 & 0.4 & 0.6 & 0.2 & 0.4 & 0.6 \\
\midrule
\multirow{4}{*}[1ex]{(a)} & Final-Layer Logits & 66.76 & 60.00 & 45.47 & 67.41 & 62.13 & 54.72 & 67.98 & 59.41 & 46.65 \\
& Pseudo-Soft-Labels & 64.81 & 51.70 & 28.32 & 64.60 & 51.61 & 32.51 & 66.81 & 58.91 & 40.71 \\
& No Convex Combination & 66.13 & 59.56 & 47.12 & 68.57 & 65.51 & 60.40 & 69.21 & 62.73 & \textbf{55.15} \\
& Ours & \textbf{68.47} & \textbf{62.36} & \textbf{50.74} & \textbf{69.76} & \textbf{66.76} & \textbf{63.22} & \textbf{70.64} & \textbf{63.58} & 54.97 \\
\midrule
\multirow{2}{*}[0.5ex]{(b)} & No Convex Combination & 55.08 & 43.85 & 21.54 & 58.47 & 50.08 & 40.97 & 59.97 & 48.60 & 41.54 \\
& Ours & \textbf{65.56} & \textbf{58.54} & \textbf{43.50} & \textbf{67.17} & \textbf{60.52} & \textbf{53.62} & \textbf{66.25} & \textbf{55.43} & \textbf{46.51} \\
\bottomrule
\end{tabular}
}
\end{table*}

We have also quantified the training time for our method on a single RTX 4090 GPU. A detailed breakdown of the costs with respect to each step of our algorithm is available in \cref{app:time_breakdown_cifar}.
We compared the complete training time with the baseline methods. These results are presented in \cref{fig:cifar-time}. Our approach secures a leading position in terms of training efficiency, necessitating less time compared to meta-learning-based methods.

\paragraph{Ablation Experiments} \label{sec:cifar-ablations}
We conduct comprehensive ablation studies, presented in \cref{tab:cifar-ablations}, to empirically validate the two central design choices of our framework: (i) the use of \textbf{intermediate features} as the optimal input for the correction network $r$, and (ii) the necessity of the \textbf{robust convex combination} for stable training.

\noindent \textbf{Impact of Correction Network Input Modality.}
Our core hypothesis is that the corrector is best achieved using intermediate features, which we posit contain more stable and richer semantic information than final-layer outputs. To verify this, we test two alternatives: (1) \textit{Final-Layer Logits} (the raw, pre-softmax outputs) and (2) \textit{Pseudo-Soft-Labels} (the final softmax outputs). The results are conclusive. Using intermediate features consistently and significantly outperforms both alternatives across all noise settings. This validates our claim in \cref{sec:neural_correction} that these features provide a superior signal for learning the complex correction mapping. Notably, \textit{Final-Layer Logits} still performs much better than \textit{Pseudo-Soft-Labels}, especially under high noise. This suggests that the softmax operation may attenuate critical signals, making the raw logits a more viable (though still suboptimal) input than the final pseudo-labels.

\noindent \textbf{Efficacy of Robust Convex Combination.}
This mechanism is our primary defense against the error amplification and model collapse that plague traditional self-training. We test its importance by comparing our full method against an ablation, \textit{No Convex Combination}, which naively uses only the single most recent correction $\ybar_t$ as the training target. The results demonstrate that this convex combination is critical for robustness, especially when supervision is scarce.

This effect is most pronounced in the data-scarce regime: without the combination, performance {collapses dramatically. In contrast, our full method remains stable. Even with a larger clean set, the convex combination provides consistent, significant gains. This confirms that our theoretically-grounded strategy of ensembling historical knowledge is essential for mitigating the risk of destabilization from a single, suboptimal correction network, thereby ensuring a stable training process.

\subsection{Experiments on Clothing1M}
Clothing1M \cite{clothing} is a real-world large-scale dataset where image instances come from online shopping websites with noisy labels derived from user tags. For these experiments, we used ResNet-50 pre-trained on ImageNet as our backbone, following previous works. The model was trained on a single RTX 4090 GPU for 15 epochs using SGD optimizer with momentum 0.9, weight decay $5e^{-4}$, and batch size 256. We set the initial learning rate to 0.01 with step decay at epochs 5 and 10. 
The training pipeline included more data augmentations: resized crop, horizontal flip, random color jitter, and random grayscale conversion following previous works.
All label correction parameters remain consistent with previous experiments, maintaining the loss balancing factor $\lambda = 0.5$. We adjusted the correction function update frequency to every 3 epochs due to the dataset's scale. In addition to the baseline methods, we also consider another leading approach DMLP \cite{dmlpdividemix} for comparison, despite it requires additional self-supervised learning procedures.
To investigate the impact of clean data quantity, we also conduct experiments with two other different clean sample set sizes (10\% and 50\%), revealing insights about our method's robustness to varying levels of supervision quality.

\begin{table}[t]
\begin{center}
\centering
\caption{Comparison of test accuracy (\%) on Clothing1M with advanced methods utilizing the clean data (or its subset). The symbol \smark denotes the use of meta-knowledge, \cmark\ denotes that the method is directly trained on clean instances with full forward and back propagation on the classifier parameters, and $\blacklozenge$ denotes methods that are trained with additional self-supervised learning procedures.}
\label{tab:clothing-main}
\begin{tabular}{lcc}
    \toprule
    {Method} & {Extra data} & {Accuracy (\%)} \\
    \midrule
    Cross-entropy & \xmark & 69.21 \\ 
    \midrule
    FasTEN & \cmark & 77.83 \\
    \midrule
    MW-Net & \smark & 73.72 \\
    EMLC$^\blacklozenge$ & \smark & 79.35 \\
    DMLP$^\blacklozenge$ (10\% clean) & \smark & 75.50 \\
    DMLP$^\blacklozenge$ (50\% clean) & \smark & 77.30 \\
    DMLP$^\blacklozenge$ & \smark & 77.31 \\
    \midrule
    Ours (10\% clean) & \smark & 78.17 \\
    Ours (50\% clean) & \smark & 79.92 \\
    Ours & \smark & \textbf{80.23} \\
    \bottomrule
\end{tabular}
\end{center}
\vskip -0.1in
\end{table}

\paragraph{Results}
The comparison results are presented in \cref{tab:clothing-main}. We also illustrate our training dynamics in \cref{app:clothing_dynamic} for clear visualizations .
Our method is categorized as ``\smark'' in the table, because it leverages the clean data as meta-knowledge for label correction and  \textit{is trained from the corrected noisy labels}, differentiating this from \textit{direct training on such clean instances}. Our approach achieves the highest accuracy of 80.23\%, outperforming other leading methods, even when these competing methods incorporate additional self-supervised learning techniques. 
Furthermore, when utilizing only 10\% of the clean data, our approach still surpasses DMLP, demonstrating that our method is capable of effectively utilizing limited clean data without self-supervised learning procedure, achieving competitive performance even with smaller subsets. 
These findings indicate that our method offers a more reliable strategy for learning with noisy labels in real-world, large-scale scenarios.

\paragraph{Training Efficiency Comparison} 
\label{app:clothing_resource_comparison}

\begin{table}[t]
\caption{Comparison of average memory usage and training time per epoch for our method, L2B, and EMLC on Clothing1M under the specified training configurations. EMLC's memory usage denotes the sum from its 2-GPU setup. We calculate our training time as the average (avg) by dividing the total training time by the complete 15 epochs of our experiments.}
\label{tab:clothing-resource}
\begin{center}
\begin{small}
\begin{tabular}{lcc}
\toprule
{Method} & \multicolumn{1}{c}{\makecell{GPU Memory\\ Cost (MiB)}} & \multicolumn{1}{c}{\makecell{Training Time\\per Epoch (hours)}} \\
\midrule
L2B & \multicolumn{1}{c}{24120} & \multicolumn{1}{c}{3.47} \\
EMLC & \multicolumn{1}{c}{22190 + 21813} & \multicolumn{1}{c}{3.00} \\
\midrule
\multirow{2}{*}[0.5ex]{Ours} & \multicolumn{1}{c}{\textbf{16915}\vspace{-2pt}} &  \multicolumn{1}{c}{\textbf{0.55}\vspace{-2pt}} \\
 &  \multicolumn{1}{c}{\vspace{-1pt}\scriptsize{(no additional cost)}} & \multicolumn{1}{c}{\vspace{-1pt}\scriptsize{(avg)}}  \\
\bottomrule
\end{tabular}
\end{small}
\end{center}
\end{table}

In order to validate our potential in practical usage, we test the training efficiency for the advanced methods, mainly the meta-learning-based methods, on Clothing1M. We set up comparative analysis of our method, L2B and EMLC regarding their time cost and performance.
We unify the classifier backbone as ResNet-50. Our method and L2B utilize a single RTX 4090 GPU with 128 mini-batches, whereas EMLC employs a 2-GPU setup, with each GPU processing 64 mini-batches. This configuration ensures a fair basis for comparison.

The training resource comparison is presented in \cref{tab:clothing-resource}.
Notably, our framework demonstrated superior efficiency and less memory, with EMLC consuming more than twice the GPU memory of ours.  
This advantage stems directly from our decoupled optimization of the classifier and the correction modules, highlighting the advantages of our approach in terms of performance, time efficiency, and GPU memory usage, underscoring its potential for practical large-scale applications.
    
\section{Conclusion}
\label{sec:conclusion}

In this paper, we proposed a novel self-training label correction framework to address learning with noisy labels. Our approach enables a classifier and a label correction function to co-evolve, guided by a small clean dataset to prevent error accumulation and enhance correction quality. We introduce noisy posterior simulation to bridge distributions between clean data and noisy labels, and leverage intermediate features for finer-grained label correction. Furthermore, a robust convex combination strategy, theoretically guaranteed for stability, is used for label updates. Empirical experiments on benchmarks like CIFAR and Clothing1M demonstrate that our method significantly outperforms current advanced techniques while also achieving higher training efficiency. This work offers an effective and robust solution for tackling large-scale noisy data with limited clean supervision in practical settings.  

{
    \small
    \bibliographystyle{ieeenat_fullname}
    \bibliography{main}
}

\clearpage
\setcounter{page}{1}
\maketitlesupplementary
\section{Further Details on Our Method}
\label{app:method}
\subsection{Considerations for the Correction Neural Network}
\label{app:correction_network}

Achieving robust generalization with the neural correction network $r$, especially when the guiding clean dataset $\Dmeta$ is of limited size, necessitates careful design of both its input modalities and structural complexity. The characteristics of $\Dmeta$, particularly its size, critically dictate these choices to ensure generalization.

When $\Dmeta$ is small, the primary concern is mitigating overfitting. This is addressed through two main strategies: architectural simplicity and a constrained input space. For the former, $r$ is intentionally implemented as a shallow network (e.g., a 2-layer MLP). For the latter, our primary method utilizes only the most recent feature embedding, $z_t = e(x; \theta_t^e)$, as input alongside the noisy posterior. This design significantly reduces the input dimensionality (to $c + \fdim$) and the network's parameter count, which is critical for effective generalization from a limited meta-set.

An alternative approach, however, is to employ an ensemble of historical feature representations, $\{z_k\}_{k=0}^{K}$. While this substantially increases the input dimensionality (to $c + (K+1) \cdot \fdim$), it could potentially offer greater stability and resistance to fluctuations in the feature extractor's representations across iterations. This strategy is generally more viable when $\Dmeta$ is sufficiently large to learn these more complex relationships without overfitting.

Beyond feature representations, other information modalities can serve as supplementary inputs:
\begin{itemize}
    \item \textbf{Final layer logits:} Raw logits from the primary classifier $f(x; \theta)$ or the noisy classifier $f_n(x; \theta)$ offer a rich, pre-softmax representation.
    \item \textbf{Pseudo-labels:} Labels from earlier training stages or auxiliary models can serve as useful signals, which have been widely used in LNL context.
    \item \textbf{Other Prior Knowledge:} Domain-specific information or encoded prior knowledge relevant to the noise or true label characteristics.
\end{itemize}

A practical methodology is to train multiple correction networks with various input configurations and select the one yielding the highest performance on $\Dvalm$ as the final predictor. To further bolster generalization, especially with a smaller $\Dmeta$, this validation of $r$ on the held-out set $\mathcal{D}_{\text{valm}}$ remains essential for selecting a model $\phi^*$ that performs robustly on unseen clean data.

\subsection{Elaboration on Robust Convex Combination}
\label{app:convex_combination}

The main text introduces a robust convex combination strategy to generate corrected labels $\yhat_{t+1}$. This section reiterates the proposition regarding its risk-non-increasing nature and further discusses its implications.

\setcounter{proposition}{0}

\begin{proposition} \label{prop:non_increasing_risk_app}
Let $\hat{y}_K^* = \sum_{k=0}^K \omega_k^* \bar{y}_k$ for $K \geq 0$ (with $\ybar_0 \equiv \ytilde$), where the weight vector $\omega^* \in \Delta^K$ minimizes the expected risk $\mathcal{R}(\hat{y}_K^*)$. Then:
\begin{align}
\mathcal{R}(\hat{y}_K^*) &\leq \min_{0 \leq k \leq K} \mathcal{R}(\bar{y}_k), \\
\mathcal{R}(\hat{y}_{K}^*) &\leq \mathcal{R}(\hat{y}_{K-1}^*) \quad \text{for all} \quad K \geq 1.
\end{align}
That is, the expected risk $\mathcal{R}(\hat{y}_K^*)$ is bounded by the minimum risk of any individual component (including the original noisy label $\ybar_0$) and does not increase as $K$ increases.
\end{proposition}
\begin{proof}
We prove both inequalities by examining special cases of the weight optimization problem.
For the first inequality, consider the set of weight vectors $\{\omega^{(j)}\}_{j=0}^K$ where $\omega^{(j)}$ places all weight on the $j$-th component (i.e., $\omega_{j}^{(j)} = 1$ and all other entries are zero). These vectors yield $\hat{y}_K^{(j)} = \bar{y}_j$ for each $j \in \{0, \ldots, K\}$. Since $\omega^*$ is optimal:
\begin{align}
\mathcal{R}(\hat{y}_K^*) \leq \min_{0 \leq k \leq K} \mathcal{R}(\bar{y}_k).
\end{align}

For the second inequality, observe that any feasible solution for the $(K-1)$-component problem can be extended to a feasible solution for the $K$-component problem by appending a zero weight for the new component. Specifically, if $\omega_{K-1}^* = (\omega_{K-1,0}^*, \ldots, \omega_{K-1,K-1}^*)$ is optimal for $K-1$, then $\omega^{(K-1)} = (\omega_{K-1,0}^*, \ldots, \omega_{K-1,K-1}^*, 0)$ is feasible for $K$ and yields $\hat{y}_K^{(K-1)} = \hat{y}_{K-1}^*$. By optimality of $\omega_K^*$:
\begin{align}
\mathcal{R}(\hat{y}_{K}^*) \leq \mathcal{R}(\hat{y}_K^{(K-1)}) = \mathcal{R}(\hat{y}_{K-1}^*).
\end{align}

Both results hold for any risk function $\mathcal{R}$, as they rely only on the structure of the feasible region and the principle of optimality.
\end{proof}

\cref{prop:non_increasing_risk_app} provides crucial theoretical support for the robustness of the proposed label correction mechanism. Its implications are multifaceted:
\begin{enumerate}
    \item \textbf{Bounded Risk:} The first inequality, $\mathcal{R}(\hat{y}_K^*) \leq \min_{0 \leq k \leq K} \mathcal{R}(\bar{y}_k)$, guarantees that the expected risk of the optimally combined corrected label $\hat{y}_K^*$ is no worse than the minimum risk of any individual component, including the original noisy label $\ybar_0 \equiv \ytilde$ or any single prediction $\bar{y}_k$.
    \item \textbf{Monotonically Non-Increasing Risk:} The second inequality, $\mathcal{R}(\hat{y}_{K}^*) \leq \mathcal{R}(\hat{y}_{K-1}^*)$ for $K \geq 1$, demonstrates that as more historical predictions $\bar{y}_k$ are incorporated into the ensemble, the expected risk of the resulting optimally weighted label $\hat{y}_K^*$ does not increase. This is a key factor for the stability of the iterative self-training process.
\end{enumerate}
These properties are vital for preventing error amplification, especially when individual corrected predictions $\bar{y}_k$ might be suboptimal. The optimization of weights $\omega^*$ on $\mathcal{D}_{\text{valm}}$ (as per \cref{eq:weight_optimization}) allows the framework to adaptively assign importance to each component, effectively down-weighting less reliable or even detrimental predictions. It is important to recognize that the strength of this framework lies in its ability to handle diverse qualities of $\bar{y}_k$ components. For instance, even if a particular $\bar{y}_k$ represents a very uncertain or "average" prediction (e.g., a uniform distribution across classes), the optimization of $\omega^*$ will determine its appropriate contribution. If incorporating such a component, despite its inherent uncertainty, aids in regularizing or smoothing the final $\hat{y}_K^*$ relative to other potentially overconfident or erroneous labels (like $\ybar_0 \equiv \tilde{y}$ or other $\bar{y}_j$), it can be assigned a beneficial weight $\omega_k^*$. This means the convex combination, guided by the risk minimization objective of \cref{prop:non_increasing_risk_app}, inherently performs a type of label smoothing. This reduces the variance of the effective training targets by producing a more stable and less noisy signal, thereby contributing to the overall robustness and performance of the learning process. The resulting $\hat{y}_K^*$ thus not only benefits from a bounded and non-increasing risk profile but also from these implicit regularization effects.

\subsection{Computational Considerations for Correction Mechanism Optimization}
A pertinent consideration is the computational cost associated with optimizing the parameters of the label correction mechanism (i.e., $\phi$ for the neural corrector $r$ and $\omega$ for the convex combination, as detailed in \cref{sec:neural_correction} and \cref{sec:convex_combination} of the main paper). Although this introduces an inner optimization loop within the broader self-training iterations, the incurred time overhead is minimal. This efficiency is largely attributable to two key aspects: first, the clean meta-dataset $\Dmeta$, on which these parameters are optimized, is typically substantially smaller than the primary noisy training dataset $\Dtrain$. Second, the correction network $r$ often operates on feature embeddings $z$ (with dimensionality $\fdim$), which are of significantly lower dimensionality than the raw input data $x$ (with dimensionality $d$). Consequently, the optimization steps for $r$ and $\omega$ are computationally inexpensive. As empirically validated in our experiments, our overall approach demonstrates favorable computational efficiency, particularly when contrasted with some meta-learning strategies that may involve more complex and resource-intensive computations, such as those requiring second-order derivatives or extensive unrolled optimization paths.

\subsection{Summary of Methodological Advantages}
\label{app:method_advantages}

Our approach, when compared to existing Learning with Noisy Labels (LNL) methods, particularly those leveraging a small clean dataset, offers several distinct advantages:

\paragraph{Computational Efficiency}
The decoupled optimization strategy employed in our framework circumvents the need for computationally expensive operations, such as the high-order derivative calculations often required by meta-gradient techniques. This design, complemented by feature reuse from the primary classifier and potentially less frequent updates to the correction mechanism, substantially minimizes computational overhead. Consequently, our method exhibits enhanced scalability and is well-suited for deployment in distributed training environments.

\paragraph{Flexibility and Adaptability}
The neural correction network, $r$, is inherently capable of modeling complex, instance-dependent noise distributions, thereby surpassing the limitations of methods reliant on simpler, often global, noise assumptions (e.g., noise transition matrices). This architectural flexibility allows our framework to be potentially adapted to more diverse problem settings, including continuous or multi-label classification tasks where transition matrices are typically less suitable. Furthermore, the training of the correction network $r$ is designed to be lightweight, facilitating its use as a plug-and-play module and readily allowing the incorporation of new domain knowledge for enhanced label refinement.

\paragraph{Robustness}
The synergistic interplay between the learned neural corrector $r$ (designed to offer refined, potentially lower-variance label predictions) and the generalized robust convex combination strategy (as detailed in \cref{sec:convex_combination} of the main paper) fosters stable and reliable learning dynamics. This robustness is not merely heuristic; it is underpinned by theoretical risk guarantees (\cref{prop:non_increasing_risk_app} ) and is continually fortified by the iterative improvement mechanism that is anchored to the ground-truth information within the clean dataset.

\begin{table*}[t]
\caption{Time consumption breakdown (in hours) of our method's different phases on CIFAR datasets using a single RTX 4090.}
\label{tab:cifar-time}
\begin{center}
\renewcommand{\arraystretch}{1.2}
\resizebox{0.8 \textwidth}{!}{%
\begin{tabular}{c|ccccc|c}
\toprule
{Phase} & \makecell{Classifier\\Training} & \makecell{Feature Extraction \\ \& Simulation} & \makecell{Train Neural \\ Label Correction} & \makecell{Train Convex \\ Combination} & \makecell{Noisy Label\\Update} & Total \\
\midrule
{Total Time} & 0.327 & 0.014 & 0.004 & 0.001 & 0.002 & \multirow{2}{*}{{0.348}} \\
{\small(Execution Times)} & {\small(100×)} & {\small(12×)} & {\small(12×)} & {\small(12×)} & {\small(12×)} & \\
\bottomrule
\end{tabular}%
}
\end{center}
\end{table*}

\begin{table*}[t]
\centering
\setlength{\tabcolsep}{6.8pt} 
\renewcommand{\arraystretch}{0.98} 
\caption{Test accuracy (\%) of all competing methods on CIFAR-100 with 1000 clean samples and 49000 noisy samples under different noise types with different noise levels. The best results are highlighted in bold.}
\label{tab:main_results_1000}
\resizebox{0.9\textwidth}{!}{%
\begin{tabular}{l@{\hspace{0.5em}}l@{\hspace{0.5em}}c|ccc|ccc|ccc}
\toprule
\multirow{3}{0.085 \textwidth}[0.6ex]{\flushright Dataset} & \multicolumn{2}{c|}{\multirow{3}{0.085 \textwidth}[0.6ex]{\flushright Method}} & \multicolumn{3}{c|}{Symmetric} & \multicolumn{3}{c|}{Asymmetric} & \multicolumn{3}{c}{Instance-dependent} \\
\cmidrule{4-12}
& & & \multicolumn{3}{c|}{$\eta$} & \multicolumn{3}{c|}{$\eta$} & \multicolumn{3}{c}{$\eta$} \\
\cmidrule{4-12}
& & & 0.2 & 0.4 & 0.6 & 0.2 & 0.4 & 0.6 & 0.2 & 0.4 & 0.6 \\
\midrule
\multirow{12}{*}[-2ex]{CIFAR-100} & \multirow{2}{*}{Cross-Entropy } & Best & 67.96 & 63.70 & 53.92 & 70.57 & 66.48 & 57.17 & 63.09 & 47.32 & 27.30 \\
& & Last & 62.88 & 48.12 & 29.72 & 63.60 & 48.58 & 32.32 & 59.87 & 41.41 & 22.59 \\
\cmidrule{2-12}
& \multirow{2}{*}{FasTEN} & Best & 67.03 & 55.87 & 51.47 & 64.79 & 60.77 & \textbf{60.79} & 66.98 & 57.78 & 34.60 \\
& & Last & 61.56 & 41.95 & 28.71 & 58.10 & 46.09 & 45.32 & 60.85 & 43.07 & 13.03 \\
\cmidrule{2-12}
& \multirow{2}{*}{MW-Net} & Best & \textbf{70.36} & \textbf{63.73} & \textbf{55.46} & \textbf{71.53} & \textbf{67.64} & 58.62 & 65.65 & 49.39 & 27.50 \\
& & Last & 63.42 & 48.50 & 30.55 & 64.30 & 49.91 & 32.16 & 61.38 & 42.06 & 23.38 \\
\cmidrule{2-12}
& \multirow{2}{*}{L2B} & Best & 63.93 & 53.24 & 23.14 & 63.80 & 57.83 & 42.85 & 64.86 & \textbf{58.79} & 45.40 \\
& & Last & 63.74 & 53.08 & 22.98 & 63.56 & 57.53 & 42.83 & 64.64 & \textbf{58.79} & 45.35 \\
\cmidrule{2-12}
& \multirow{2}{*}{{Ours}} & Best & 68.13 & 61.74 & 47.91 & 69.92 & 66.99 & 58.58 & \textbf{67.42} & 56.08 & \textbf{48.09} \\
& & Last & \textbf{65.56} & \textbf{58.54} & \textbf{43.50} & \textbf{67.17} & \textbf{60.52} & \textbf{53.62} & \textbf{66.25} & 55.43 & \textbf{46.51} \\
\bottomrule
\end{tabular}
}
\end{table*}

\section{Further Details on the Experiments}
\label{app:experiments}

We will publicly release our code and implementation details upon acceptance of this paper. The repository will include all experiments conducted on CIFAR and Clothing1M datasets, enabling full reproducibility of our results. In the following sections, we provide additional experimental details, comprehensive analyses, and extended results that complement the main paper.

We detail the specific implementation aspects and hyper-parameters used in our experiments across various datasets. \cref{tab:experiment-details} outlines the key training parameters and settings employed in our experiments.

\subsection{Additional Experiments on CIFARs} 
\label{app:time_breakdown_cifar}
Our approach consists of five main phases: (1) \textbf{Classifier Training}, where we update the model parameters using corrected labels; (2) \textbf{Feature Extraction \& Noisy Posterior Simulation}, which involves extracting feature representations from both datasets and simultaneously generating the simulated noisy posteriors on the clean dataset; (3) \textbf{Neural Correction Optimization}, which includes training and validating the correction network; (4) \textbf{Convex Combination}, where we compute optimal weights for combining noisy labels and correction network predictions; and (5) \textbf{Label Update}, where we apply the computed weights to generate the refined corrected labels for the subsequent training iteration.

In \cref{sec:implementations}, we update the noisy posteriors, features and correction function every 5 epochs, starting after the initial 40 epochs. 
As shown in \cref{tab:cifar-time}, the computational overhead associated with our correction mechanism (phases 2-5) is not significant. This confirms that our method achieves significant performance improvements without introducing substantial computational burden. While the neural correction optimization remains computationally light, the classifier training phase dominates the overall time consumption, which is consistent with standard deep learning training procedures.

\paragraph{Ablations on the Loss Balancing Factor}
\label{app:lambda_ablation_cifar}
In addition, we also conduct experiments to examine how the loss balancing factor \(\lambda\) affects performance. Selecting the appropriate \(\lambda\) in \cref{eq:framework_train_f} is a balancing act: it should enable the classifier to estimate the noisy posterior accurately, while preventing the feature extractor from overfitting to noisy patterns. We vary \(\lambda\) across different values in \(\{0.2, 0.5, 1.0\}\), with the results summarized in \cref{tab:cifar-lambda}.
The table shows that the impact of different \(\lambda\) values is relatively modest, with performance variations typically within 1-2 percentage points across all noise types and levels.

\begin{table*}[t]
\centering
\small 
\setlength{\tabcolsep}{6.8pt} 
\renewcommand{\arraystretch}{0.98} 
\caption{Test accuracy (\%) of our method using different values of loss balancing factor $\lambda$ on CIFAR-100 under different noise types with different noise levels. The last recorded performances in the training procedure are reported.}
\label{tab:cifar-lambda}
\resizebox{0.76\textwidth}{!}{%
\begin{tabular}{c|ccc|ccc|ccc}
\toprule
& \multicolumn{3}{c|}{Symmetric} & \multicolumn{3}{c|}{Asymmetric} & \multicolumn{3}{c}{Instance-dependent} \\
\cmidrule{2-10}
\multirow{-2}{*}{$\lambda$ value} & \multicolumn{3}{c|}{$\eta$} & \multicolumn{3}{c|}{$\eta$} & \multicolumn{3}{c}{$\eta$} \\
\cmidrule{2-10}
& 0.2 & 0.4 & 0.6 & 0.2 & 0.4 & 0.6 & 0.2 & 0.4 & 0.6 \\
\midrule
$0.2$ & \textbf{68.95} & \textbf{63.54} & \textbf{52.69} & \textbf{71.30} & \textbf{68.18} & \textbf{63.78} & 70.01 & 62.84 & \textbf{55.21} \\
$0.5$ & 68.47 & 62.36 & 50.74 & 69.76 & 66.76 & 63.22 & \textbf{70.64} & \textbf{63.58} & 54.97 \\
$1.0$ & 66.35 & 59.30 & 48.74 & 69.43 & 66.41 & 62.61 & 68.83 & 60.83 & 53.96 \\
\bottomrule
\end{tabular}
}
\end{table*}
\begin{table*}[t]
\centering
\small 
\setlength{\tabcolsep}{6.8pt} 
\renewcommand{\arraystretch}{0.98} 
\caption{Test accuracy (\%) of our method using different label correction frequencies on CIFAR-100 under different noise types with different noise levels. The frequency value indicates how often (in epochs) the feature sets and correction function are updated.}
\label{tab:cifar-frequency}
\resizebox{0.8\textwidth}{!}{%
\begin{tabular}{c|ccc|ccc|ccc}
\toprule
& \multicolumn{3}{c|}{Symmetric} & \multicolumn{3}{c|}{Asymmetric} & \multicolumn{3}{c}{Instance-dependent} \\
\cmidrule{2-10}
\multirow{-2}{*}{\makecell{ Label Correction \\ Frequency}} & \multicolumn{3}{c|}{$\eta$} & \multicolumn{3}{c|}{$\eta$} & \multicolumn{3}{c}{$\eta$} \\
\cmidrule{2-10}
& 0.2 & 0.4 & 0.6 & 0.2 & 0.4 & 0.6 & 0.2 & 0.4 & 0.6 \\
\midrule
$5$ & \textbf{68.47} & \textbf{62.36} & \textbf{50.74} & 69.76 & 66.76 & \textbf{63.22} & \textbf{70.64} & \textbf{63.58} & \textbf{54.97} \\
$8$ & 68.09 & 61.82 & 50.20 & \textbf{69.94} & \textbf{67.36} & 62.99 & 69.87 & 62.71 & 53.87 \\
$10$ & 68.09 & 60.94 & 49.36 & 69.76 & 66.36 & 61.66 & 69.10 & 61.68 & 52.99 \\
\bottomrule
\end{tabular}
}
\end{table*}

\paragraph{Ablations on the Label Correction Frequency}
\label{app:frequency_ablation_cifar}
We investigate the impact of label correction frequency on model performance. The correction frequency determines how often we the correction function and the corrected labels during training. As described in \cref{sec:implementations}, our default setting updates every 5 epochs after the initial 40 epochs on CIFAR datasets. We vary this frequency across different values in \(\{5, 8, 10\}\) to examine its influence on performance. The results are summarized in \cref{tab:cifar-frequency}.
.
The experimental results demonstrate that the performance differences are relatively modest, with variations typically within 1-2 percentage points. This suggests that our method maintains robust performance across a range of correction frequencies, though more frequent updates tend to provide incremental benefits, particularly in high-noise scenarios where label quality is more critical to model convergence.

\paragraph{Complete Results Presentation}
\label{app:cifar_results_all}

We present experimental results for various datasets under different noise conditions on CIFAR-10 and CIFAR-100. \cref{fig:cifar10} and \cref{fig:cifar100} show the test accuracy trends across training epochs. The dashed vertical lines at epochs 60 and 80 indicate learning rate drop points.
The figures demonstrate the effectiveness of our method compared to existing approaches across different datasets and noise settings. As shown in the plots, our method consistently achieves higher test accuracy, especially in high noise ratio scenarios (60\%), where the performance gap is more significant. Our method exhibits more stable learning behavior and better convergence properties.

\subsection{Robustness with Scarce Clean Data}
\label{sec:cn1000}

To rigorously evaluate our framework's robustness under the extreme data-scarce scenario introduced in \cref{sec:experiments:dataset}, we conduct a comprehensive study on CIFAR-100 using only 1,000 clean samples (10 per class), establishing a 49:1 noisy-to-clean data ratio. The quantitative results are summarized in Table \ref{tab:main_results_1000}.

We compare our method against the baseline methods. Notably, EMLC is excluded from this comparison as it fails to learn in this setting without the aid of self-supervised pre-training. The result shows that our framework consistently and significantly outperforms the competing baselines across most noise types, with a particularly strong advantage under high noise rates and complex instance-dependent noise. This validates the superior efficiency and robustness of our proposed correction mechanism in leveraging very limited clean supervision.

\subsection{Training Dynamics on Clothing1M} 
\label{app:clothing_dynamic}

\cref{fig:clothing_acc} and \cref{fig:clothing_err} illustrate the performance of our method during the training on Clothing1M. In these figures, we measure the performance on the test data and the clean data in $\Dvalm$ as the correction neural network is not significantly overfitted on this validation subset.
We observe that under the guidance of our iterative label correction framework, the classifier robustly learns from the progressively refined labels. This empirically validates The stable convergence, without the error amplification common in self-training, highlights the effectiveness of our closed-loop mechanism.

\begin{figure*}[t]
    \begin{minipage}[t]{0.48\textwidth}
        \centering
        \includegraphics[width=\textwidth]{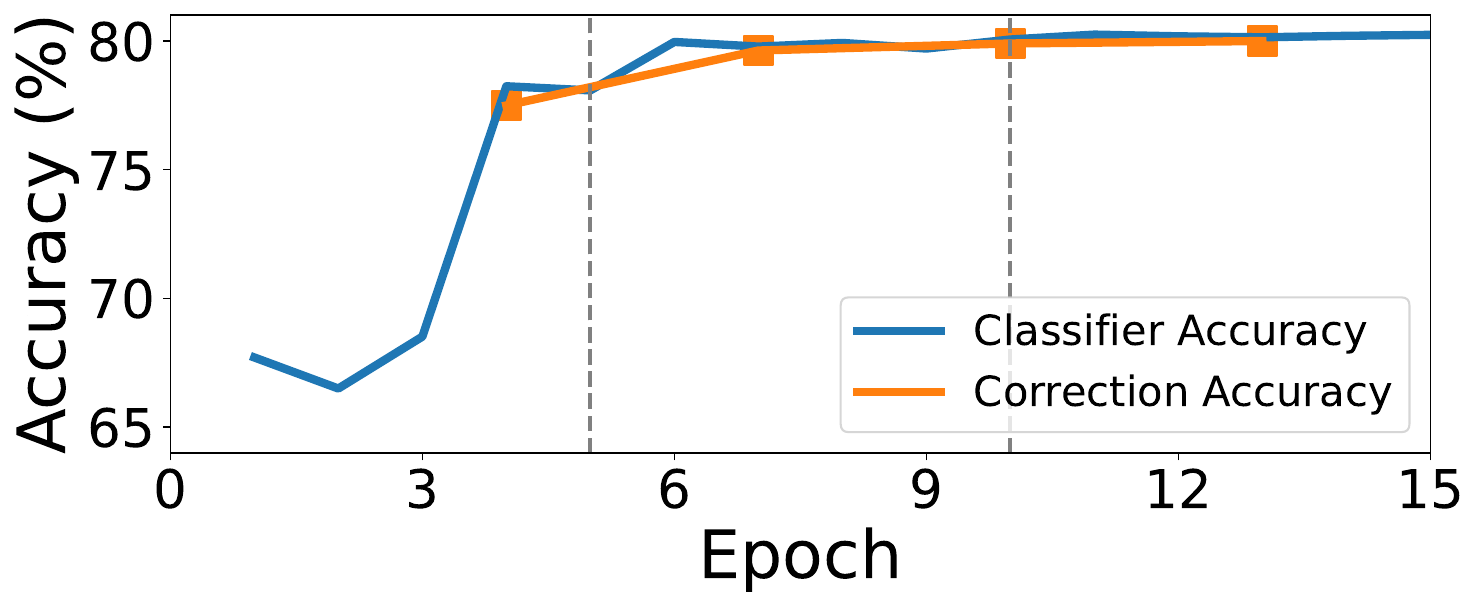} 
        \caption{Classification accuracy on test data and clean meta data during training on Clothing1M. The vertical dashed lines represent the epochs at which the learning rate decays at 5 and 10 epoch.}
        \label{fig:clothing_acc}
    \end{minipage}
    \hfill
    \begin{minipage}[t]{0.48\textwidth}
        \centering
        \includegraphics[width=\textwidth]{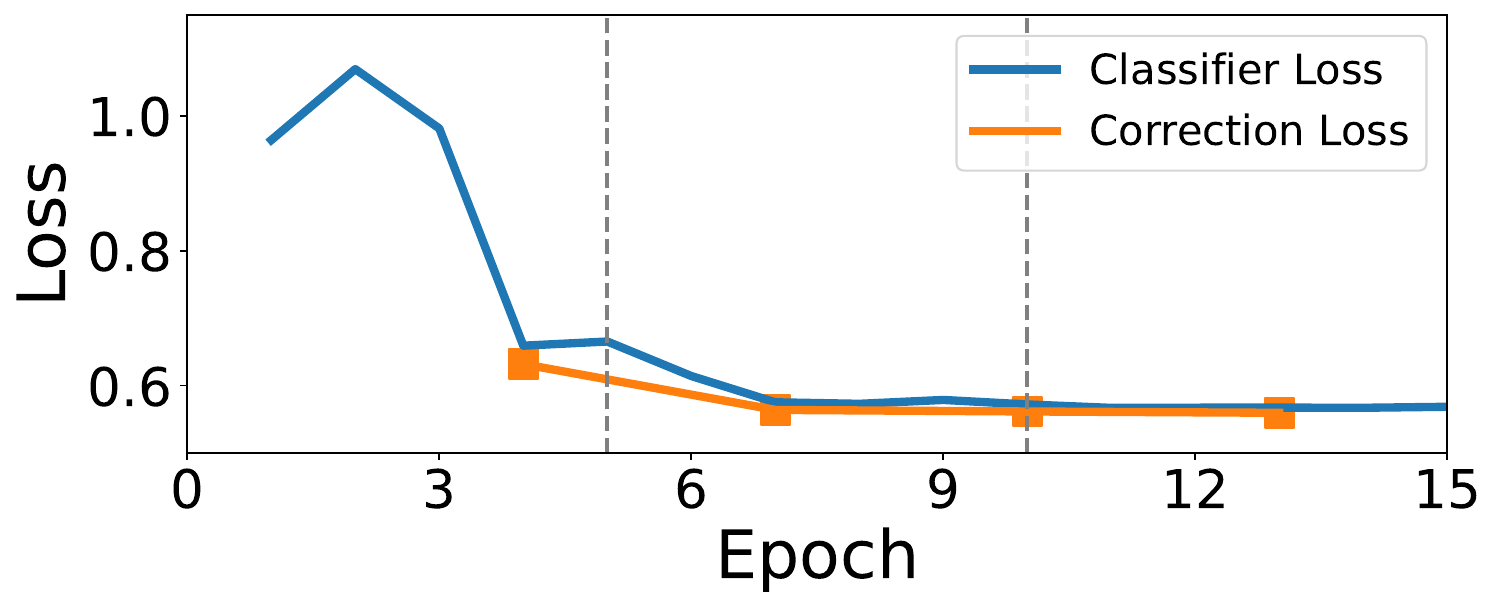}
        \caption{Classification error on test data and clean data during training on Clothing1M. The vertical dashed lines represent the epochs at which the learning rate decays at 5 and 10 epoch.}
        \label{fig:clothing_err}
    \end{minipage}
\end{figure*}

\section{Limitations and Broader Impact}
\label{app:limitations}
 
\paragraph{Limitations}
Our limitations include the dependency on a small, genuinely clean dataset $\Dmeta$ and the presence of several hyper-parameters that may require tuning for new datasets or noise characteristics. 

\paragraph{Broader Impact}
Our method offers several significant broader impacts. By effectively relabeling training samples, our method directly reduces label noise and variance within the training set, leading to more robust model training and improved generalization. This inherent label correction capability suggests that our framework could be compatible with, and potentially serve as an enhancement for, other LNL strategies, for example, by providing cleaner labels for their subsequent processing stages. The closed-loop, iterative nature of our framework, where the classifier and correction function co-evolve under the continuous guidance of clean data (via $\Dmeta$), inherently limits the accumulation of errors and reduces the risk of model collapse, a common challenge in many self-correction or pseudo-labeling schemes.

Beyond standard LNL scenarios, our approach may offer promise for certain transfer learning settings. For instance, it could potentially establish a more robust mapping between a source and a target domain by correcting noisy or mismatched labels in the source domain, guided by feedback from a small, clean dataset representative of the target domain. This would effectively refine the source data, akin to adaptively improving sample quality for more effective knowledge transfer.

\newpage

\begin{table*}[t]
\caption{Experimental details for our method across different datasets.}
\label{tab:experiment-details}
\begin{center}
\begin{small}
\resizebox{0.75\textwidth}{!}{%
\begin{tabular}{l | c c c}
    \toprule
        {Dataset} & 
        {CIFAR-10} & 
        {CIFAR-100} & 
        {Clothing1M}\\
        {Architecture} & ResNet-34 & ResNet-34 & ResNet-50 \\
        {Pretraining} & None & None & ImageNet \\
        {Training Epochs} & 100 & 100 & 15 \\
        \midrule
        {Training Augmentations} &
        \thead{Horizontal Flip \\ Random Crop} &
        \thead{Horizontal Flip \\ Random Crop} & 
        \thead{Resized Crop \\ Horizontal Flip \\
        Random ColorJitter \\ Random Grayscale
        }
        \\
        {Optimizer} & SGD & SGD & SGD \\
        {Batch Size} & $128$ & $128$ & $256$ \\
        {Momentum} & $0.9$ & $0.9$ & $0.9$ \\
        {Weight decay} & $5e^{-4}$ & $5e^{-4}$ & $5e^{-4}$ \\
        {Initial LR} & $0.1$ & $0.1$ & $0.01$ \\
        {Scheduler} & LR-step & LR-step & LR-step \\
        {Scheduler Milestone} & $60,80$ & $60,80$ & $5,10$ \\
        {Scheduler Step-Decay} & $0.1$ & $0.1$ & $0.1$ \\
        \midrule
        {Label Correction Frequency} & 5  & 5  &  3  \\
        \cdashline{2-4}
        \noalign{\vskip 0.25ex}
        {Balancing coefficient $\lambda$} & \multicolumn{3}{c}{0.5} \\
        {Clean Dataset Split Ratio} & \multicolumn{3}{c}{0.8 : 0.2} \\
        {Correction Network Type} & \multicolumn{3}{c}{2-layer MLP} \\
        {Correction Network Hidden Units} & \multicolumn{3}{c}{256} \\
        {Correction Network Activation} & \multicolumn{3}{c}{ReLU} \\
        {Correction Network Optimizer} & \multicolumn{3}{c}{SGD} \\
        {Correction Network LR} & \multicolumn{3}{c}{$1e^{-3}  \rightarrow 1e^{-4}$} \\
    \bottomrule
\end{tabular}
}
\end{small}
\end{center}
\end{table*}

\newpage

\newpage

\begin{figure*}[htbp]
\centering 
\includegraphics[width=.95\textwidth]{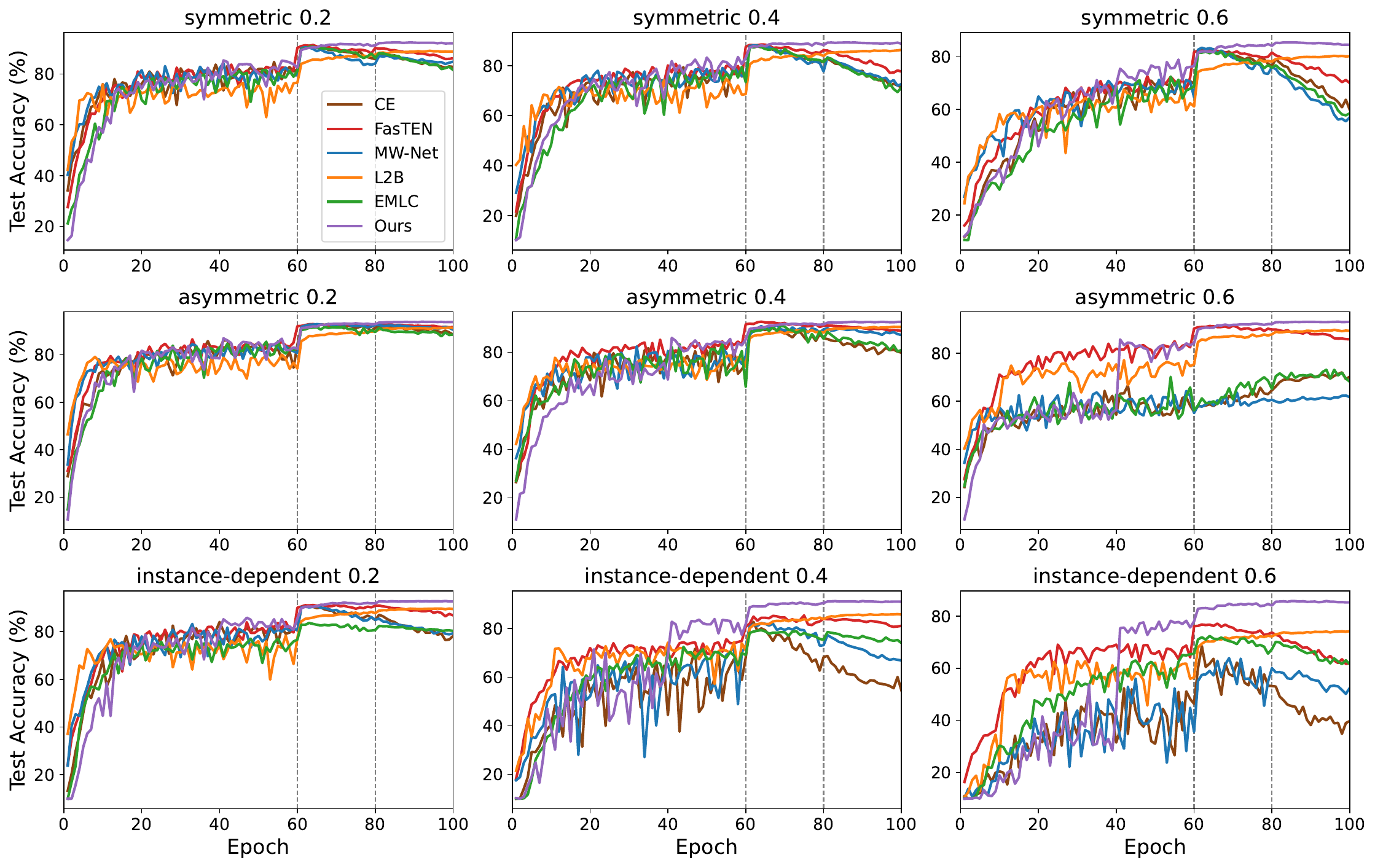} 
\caption{Training dynamics of all methods on CIFAR-10 under different noise scenarios.}
\label{fig:cifar10}
\end{figure*}

\begin{figure*}[htbp]
\centering 
\includegraphics[width=.98\textwidth]{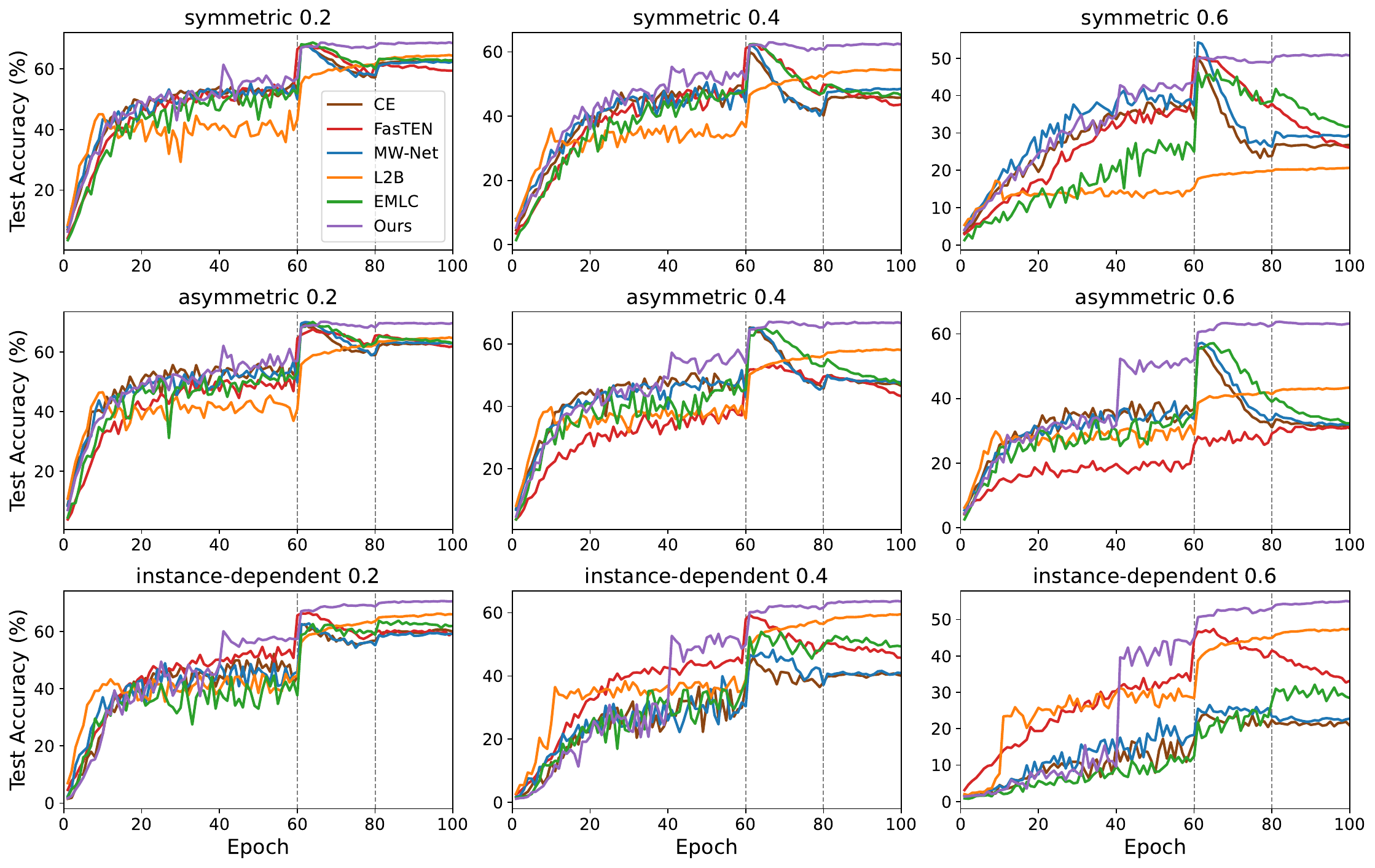} 
\caption{Training dynamics of all methods on CIFAR-100 under different noise scenarios.}
\label{fig:cifar100}
\end{figure*}

\end{document}